\documentclass[twocolumn, switch]{article}

\usepackage{preprint}

\usepackage{graphicx}
\usepackage{subcaption}
\usepackage{float}

\usepackage{amsmath, amsthm, amssymb, amsfonts}
\newtheorem{theorem}{Theorem}
\newtheorem{corollary}{Corollary}

\usepackage[numbers,square]{natbib}
\usepackage[utf8]{inputenc}	
\usepackage[T1]{fontenc}
\usepackage{xcolor}	
\usepackage[colorlinks = true,
            linkcolor = purple,
            urlcolor  = blue,
            citecolor = cyan,
            anchorcolor = black]{hyperref}
\usepackage{booktabs}
\usepackage{nicefrac}
\usepackage{microtype}
\usepackage{lineno}
\usepackage{float}

\usepackage{newfloat}
\DeclareFloatingEnvironment[name={Supplementary Figure}]{suppfigure}
\usepackage{sidecap}
\sidecaptionvpos{figure}{c}

\usepackage{titlesec}

\titleformat{\section}
{\normalfont\large\bfseries\scshape}
{\thesection}{0.6em}{}

\titlespacing\section{0pt}{12pt plus 3pt minus 3pt}{1pt plus 1pt minus 1pt}
\titlespacing\subsection{0pt}{10pt plus 3pt minus 3pt}{1pt plus 1pt minus 1pt}
\titlespacing\subsubsection{0pt}{8pt plus 3pt minus 3pt}{1pt plus 1pt minus 1pt}

\usepackage{tikz,xcolor,hyperref}

\definecolor{lime}{HTML}{A6CE39}
\DeclareRobustCommand{\orcidicon}{
	\begin{tikzpicture}
	\draw[lime, fill=lime] (0,0)
	circle [radius=0.16]
	node[white] {{\fontfamily{qag}\selectfont \tiny ID}};
	\draw[white, fill=white] (-0.0625,0.095)
	circle [radius=0.007];
	\end{tikzpicture}
	\hspace{-2mm}
}
\foreach \x in {A, ..., Z}{\expandafter\xdef\csname orcid\x\endcsname{\noexpand\href{https://orcid.org/\csname orcidauthor\x\endcsname}
			{\noexpand\orcidicon}}
}

\title{When Does Pairing Seeds Reduce Variance? Evidence from a Multi-Agent Economic Simulation}

\usepackage{xwatermark}
\newwatermark[firstpage,color=gray!90,angle=0,scale=0.28, xpos=0in,ypos=-5in, ]{*correspondence: \texttt{uditsharma.iitkgp@gmail.com}}

\usepackage{authblk}
\author[]{Udit Sharma\orcidA{}}

\begin{document}

\twocolumn[
  \begin{@twocolumnfalse}

\maketitle

\begin{abstract}
	
	Machine learning systems appear stochastic but are deterministically random, as seeded pseudorandom number generators produce identical realisations across repeated executions. Learning-based multi-agent simulations are increasingly used to compare algorithms, design choices, and policy interventions under such stochastic dynamics, but comparative evaluation frequently exhibits high variance due to random initial conditions, environment noise, and learning stochasticity. Standard evaluation practice typically treats runs across alternatives as independent and does not exploit shared sources of randomness. This paper analyses the statistical structure of comparative evaluation under shared random seeds. Under this design, competing systems are evaluated using identical seeds, inducing matched stochastic realisations and yielding variance reduction whenever outcomes are positively correlated at the seed level. We characterise how this variance reduction propagates through the evaluation pipeline, leading to tighter confidence intervals, improved directional stability, higher statistical power, and substantial effective sample size gains at fixed computational budgets, while recovering standard independent evaluation when correlation is absent. We demonstrate these effects using an extended learning-based multi-agent economic simulator with an additional fiscal policy instrument, where paired evaluation exposes systematic differences in aggregate and distributional outcomes that remain statistically inconclusive under independent evaluation at fixed budgets.

\end{abstract}
\vspace{0.35cm}

  \end{@twocolumnfalse}
] 
\section{Introduction}

Learning-based simulators have become an important tool for studying multi-agent systems, supporting the evaluation of algorithms, design choices, and interventions in settings where analytic guarantees are unavailable. As such simulators and training pipelines grow more complex and computationally expensive, conclusions are increasingly drawn from a limited number of stochastic runs, placing greater weight on the reliability of comparative evaluation procedures.

A core difficulty is that comparative outcomes in learning-based simulators often exhibit substantial variance across random seeds, leading to unstable and potentially misleading comparative claims under standard evaluation practices\citep{10.5555/3540261.3542505}. Consequently, recent work has emphasised the need for more principled evaluation, including explicit treatment of variability across runs and greater attention to experimental design\citep{10.5555/3540261.3542505, colas2018randomseedsstatisticalpower, 10.5555/3504035.3504427, madhyastha-jain-2019-model, JMLR:v25:23-0183}.

This paper focuses on a complementary source of statistical efficiency in comparative evaluation that is often available in learning-based simulations but not routinely exploited. Although learning systems are typically treated as stochastic, their randomness is deterministic, governed by seeded pseudorandom number generators that reproduce identical stochastic realisations across executions. As a result, outcomes under alternative algorithms or interventions can exhibit positive correlation at the level of random seeds when they share the same underlying stochastic realisation, inducing aligned trajectories even when learned behaviour differs.

We formalise paired seed evaluation as an application of the classical paired comparison design, treating the random seed as the unit of randomisation and holding it fixed across alternative policies, with effects estimated through within-seed differences. As in standard experimental settings, positive seed-level correlation yields lower variance than independent evaluation, while zero correlation reduces to conventional independent evaluation without loss of efficiency. Accordingly, paired seed evaluation is best viewed as a design choice grounded in established experimental principles, adapted to learning-based simulators. By aligning competing systems to a common stochastic realisation, shared randomness is used to improve the reliability of empirical comparisons under fixed computational budgets.

\section{Evaluation in Learning-Based Simulators}

Across learning-based simulators and benchmarks, evaluation typically runs each algorithm or intervention under independently sampled random seeds and reports aggregate statistics\citep{10.5555/3540261.3542505,10.5555/3504035.3504427,JMLR:v25:23-0183,islam2017reproducibilitybenchmarkeddeepreinforcement}. This convention remains common in contemporary benchmark studies, including deep reinforcement learning benchmarks, multi-agent learning frameworks, and learning-based simulators for complex economic and social systems \citep{pmlr-v48-duan16,pmlr-v119-laskin20a,mi2024taxai,papoudakis2021benchmarkingmultiagentdeepreinforcement,yu2022surprisingeffectivenessppocooperative,zheng2021aieconomist}. Such evaluations are sufficient for demonstrating qualitative behaviour or feasibility.

When the goal is comparative inference under fixed computational budgets, however, this evaluation practice places strong implicit assumptions on how stochasticity is treated. In much of the literature, random seeds are used primarily to ensure reproducibility rather than as a source of exploitable statistical structure. Treating all runs as independent discards shared stochastic variation induced by the simulator, increasing uncertainty in the few-run regime. In learning-based simulators, random seeds provide a mechanism for controlling common randomness, as they jointly determine initial conditions, environment transitions, and algorithmic stochasticity.

Related pairing ideas appear in biomedicine\citep{NARIYA2023100791} and NLP\citep{peyrard-etal-2021-better}, where systems are compared on shared inputs rather than shared stochastic realisations. In learning-based simulators, paired seeds have occasionally appeared as implementation details in specific contexts\citep{immorlano2025technicalreportunifieddiffusion, wu2025sgmstatisticalgodelmachine}, but have not been studied as a general evaluation design choice.

Patterson et al.\citep{JMLR:v25:23-0183} recommend the use of paired statistical comparisons when analysing reinforcement learning experiments, framing pairing as a best practice at the analysis stage. This work takes that recommendation as a starting point and investigates when and why pairing is effective in learning-based multi-agent simulators. We characterise the statistical properties of paired seed evaluation by treating the seed as the fundamental unit of randomisation, quantify the extent of variance reduction achieved under positive seed-level correlation, and demonstrate these effects empirically. Although the underlying statistical principle is well established, its systematic application to learning-based simulators, in which seeds determine complex training trajectories rather than isolated random draws, merits explicit and careful treatment.

\section{Paired Seed Evaluation}

\subsection{Experimental Design and Estimand}

In learning-based simulators, outcomes depend jointly on policy choices and random seeds that govern initialization, environment stochasticity, and learning dynamics. We consider two policy regimes indexed by $d \in \{0,1\}$. Let $s \in S$ denote a random seed that fully determines all sources of randomness in the simulator, including initial conditions, environment stochasticity, and learning noise. For a given regime $d$ and seed $s$, let $Y(d,s)$ denote a scalar outcome of interest computed after training and evaluation. For a fixed seed $s$, outcomes under different regimes differ only through the policy intervention.

The estimand of interest is the average treatment effect

\begin{equation*}
	\Delta = \mathbb{E}_{s}\!\left[ Y(1,s) - Y(0,s) \right]
\end{equation*}

where the expectation is taken over the distribution of seeds.

Our objective is to estimate $\Delta$ with minimal variance, subject to a fixed computational budget.

\subsection{Independent and Paired Evaluation Designs}

A common evaluation strategy is to estimate $\Delta$ using independently trained runs under each regime. Under this design, outcomes for $d = 1$ and $d = 0$ are generated using independent seeds. While this estimator is unbiased, it does not exploit shared stochastic structure.

An alternative design is paired evaluation. Under this design, outcomes under both regimes are evaluated using the same seed. Intuitively, pairing holds fixed all seed-level sources of randomness unrelated to the policy intervention, thereby isolating policy induced differences.

Formally, let $\{s_i\}_{i=1}^n$ be a set of seeds. We define two unbiased estimators of $\Delta$. The independent estimator is given by

\begin{equation*}
	\hat{\Delta}_{\text{ind}}
	= \frac{1}{n} \sum_{i=1}^{n} Y\!\left(1, s_i^{(1)}\right)
	- \frac{1}{n} \sum_{i=1}^{n} Y\!\left(0, s_i^{(0)}\right)
\end{equation*}

where $s_i^{(1)}$ and $s_i^{(0)}$ are independent draws. The paired estimator is given by
\begin{equation*}
	\hat{\Delta}_{\text{pair}}
	= \frac{1}{n} \sum_{i=1}^{n} \Big( Y(1, s_i) - Y(0, s_i) \Big)
\end{equation*}

where the same seed $s_i$ is used for both regimes. Crucially, the difference between these estimators arises entirely from the evaluation design rather than from the estimand or estimator form. In particular, paired evaluation treats the random seed as the unit of randomization, holding fixed shared stochastic components by construction, whereas independent evaluation treats seeds as incidental noise and discards this structure.

\subsection{Variance Reduction under Paired Evaluation}

We now formalise the statistical advantage of paired evaluation.

\begin{theorem}
	
	Let $Y(1,s)$ and $Y(0,s)$ be square-integrable random variables. The variance of the paired estimator satisfies
	\begin{equation}
		\operatorname{Var}\!\left( \hat{\Delta}_{\text{pair}} \right)
		=
		\operatorname{Var}\!\left( \hat{\Delta}_{\text{ind}} \right)
		-
		\frac{2}{n}\,
		\operatorname{Cov}\!\left( Y(1,s), Y(0,s) \right)
		\label{eq:variance-reduction}
		\tag{1}
	\end{equation}
	
\end{theorem}

Proof can be found in Appendix \ref{app:proof_t1}.

\subsection{Interpretation and Corollary}
\label{subsec:variance-reduction-corollary}

Theorem \ref{eq:variance-reduction} establishes that paired evaluation is statistically equivalent to a common random-numbers estimator, with the extent of variance reduction governed by the covariance between outcomes across the two regimes. While common random numbers are well understood in Monte Carlo estimation\citep{law1982simulation}, their role in learning-based simulators is fundamentally different: here, random seeds index entire training trajectories and induced environments, rather than single stochastic draws.

\begin{corollary}
	If 
	\[
	\operatorname{Cov}\!\left( Y(1,s), Y(0,s) \right) > 0
	\]
	then paired evaluation strictly reduces estimator variance relative to independent evaluation.
	Moreover, letting
	\[
	\begin{aligned}
		\sigma_i &:= \sqrt{\operatorname{Var}\!\left( Y(i,s) \right)} \quad i \in \{0,1\} \\
		\rho &:= \operatorname{Corr}\!\left( Y(1,s), Y(0,s) \right)
	\end{aligned}
	\]
	the reduction in variance achieved by paired evaluation satisfies
	\begin{equation}
		\operatorname{Var}\!\left( \hat{\Delta}_{\text{ind}} \right)
		-
		\operatorname{Var}\!\left( \hat{\Delta}_{\text{pair}} \right)
		=
		\frac{2}{n}\,
		\rho\,\sigma_1\,\sigma_0
		\tag{2}
		\label{eq:corollary-positivecov}
	\end{equation}
\end{corollary}

Equation \ref{eq:corollary-positivecov} shows that paired evaluation is most effective when outcomes under alternative regimes exhibit strong seed-level correlation, as shared stochastic variation cancels in within-seed differences. In learning-based simulators, outcomes under alternative policies often exhibit positive correlation at the level of random seeds. When evaluations are conducted under a common seed, alternative regimes are exposed to the same realised stochastic components of the simulator and training pipeline. As a result, variation across seeds induces systematically aligned responses across treatments, even when learned behaviour differs.

Paired evaluation treats the random seed as the unit of randomisation, ensuring that dominant sources of simulator and training stochasticity are held fixed across regimes. In this way, shared randomness is converted from a source of evaluation noise into a controlled component of the experimental design, directly analogous to matched or blocked designs in classical experiments.

\begin{figure*}[t]
	\centering
	\subfloat[Confidence interval half-width]{
		\includegraphics[width=0.30\textwidth]{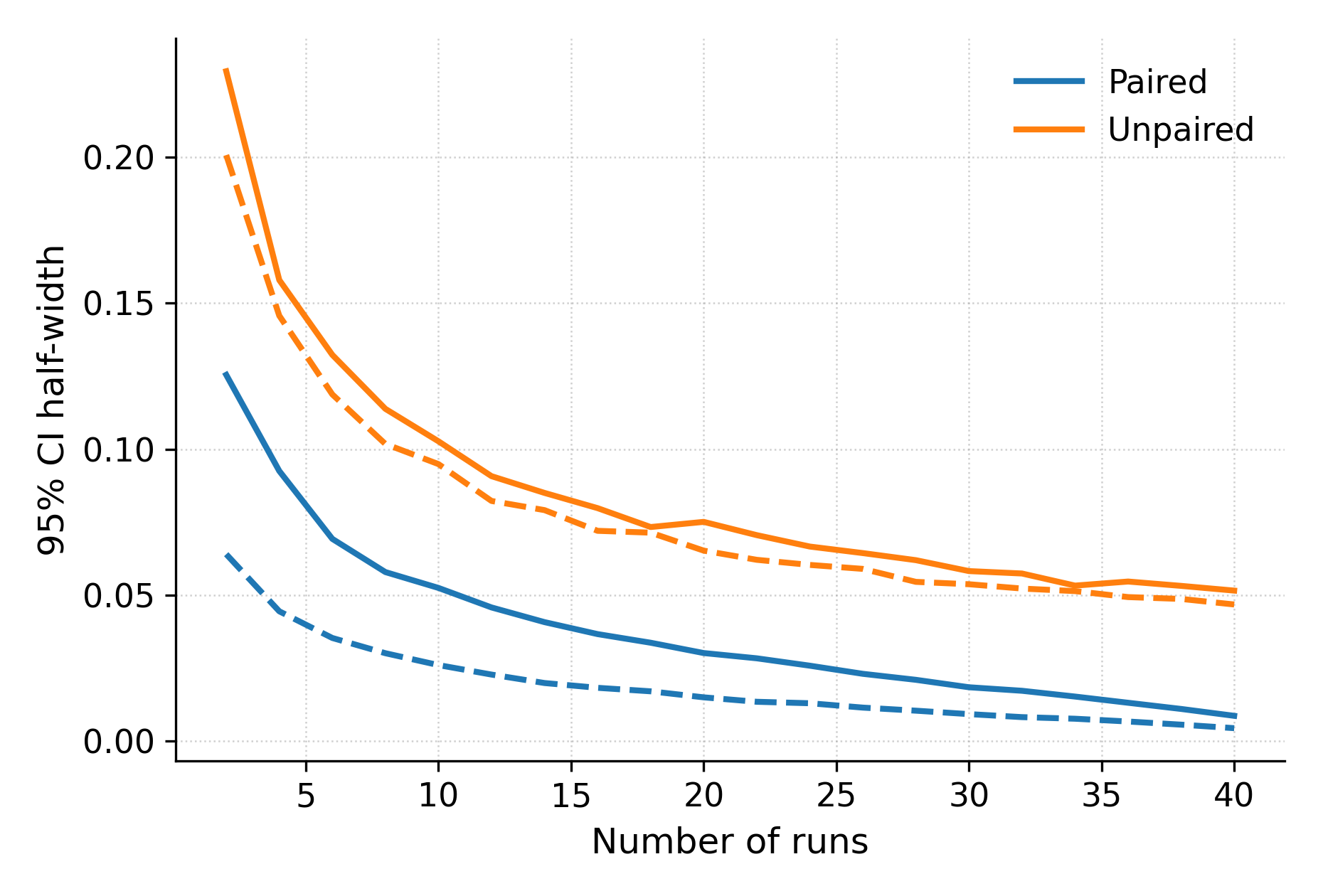}
	}
	\hfill
	\subfloat[Statistical power]{
		\includegraphics[width=0.30\textwidth]{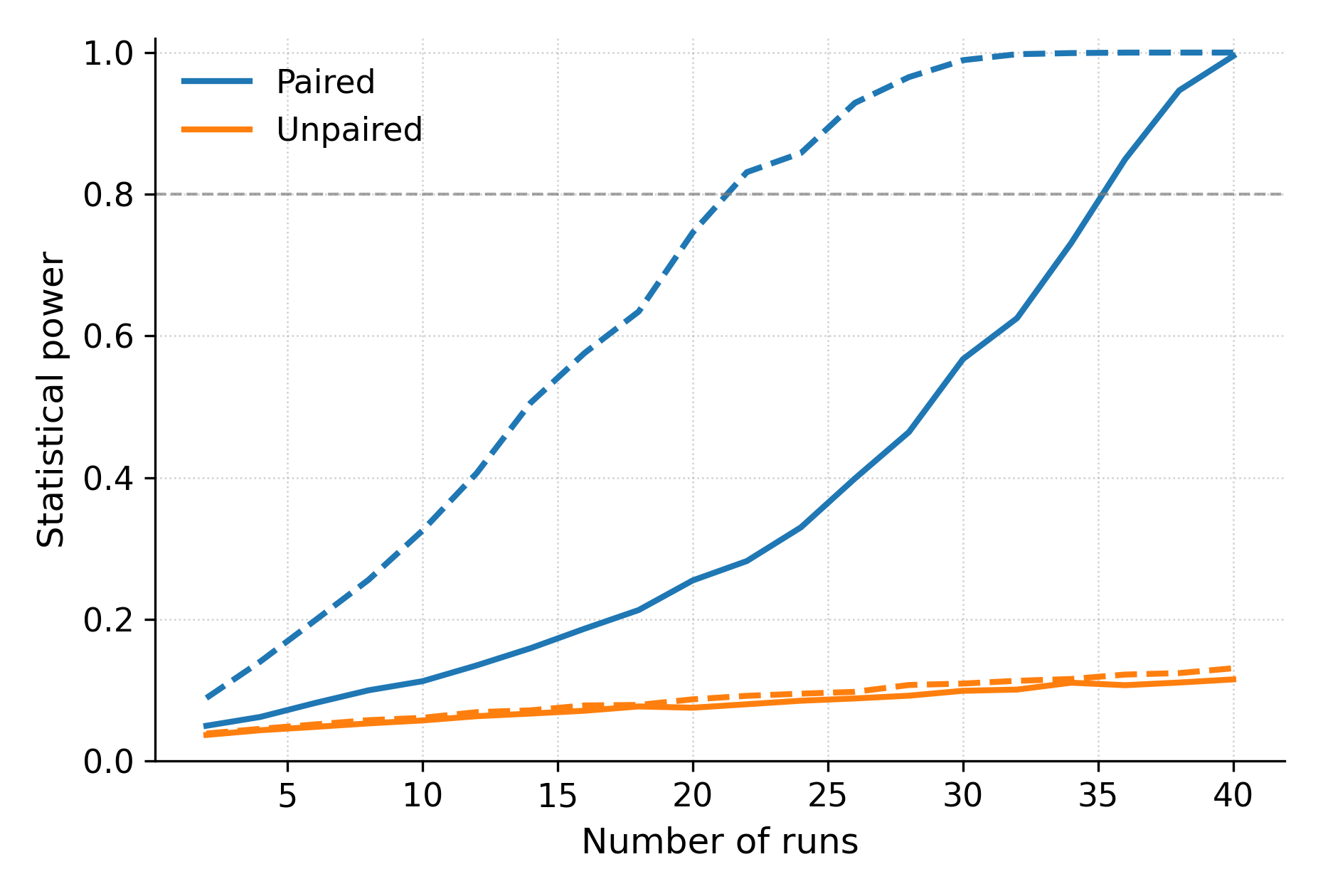}
	}
	\hfill
	\subfloat[Directional stability]{
		\includegraphics[width=0.30\textwidth]{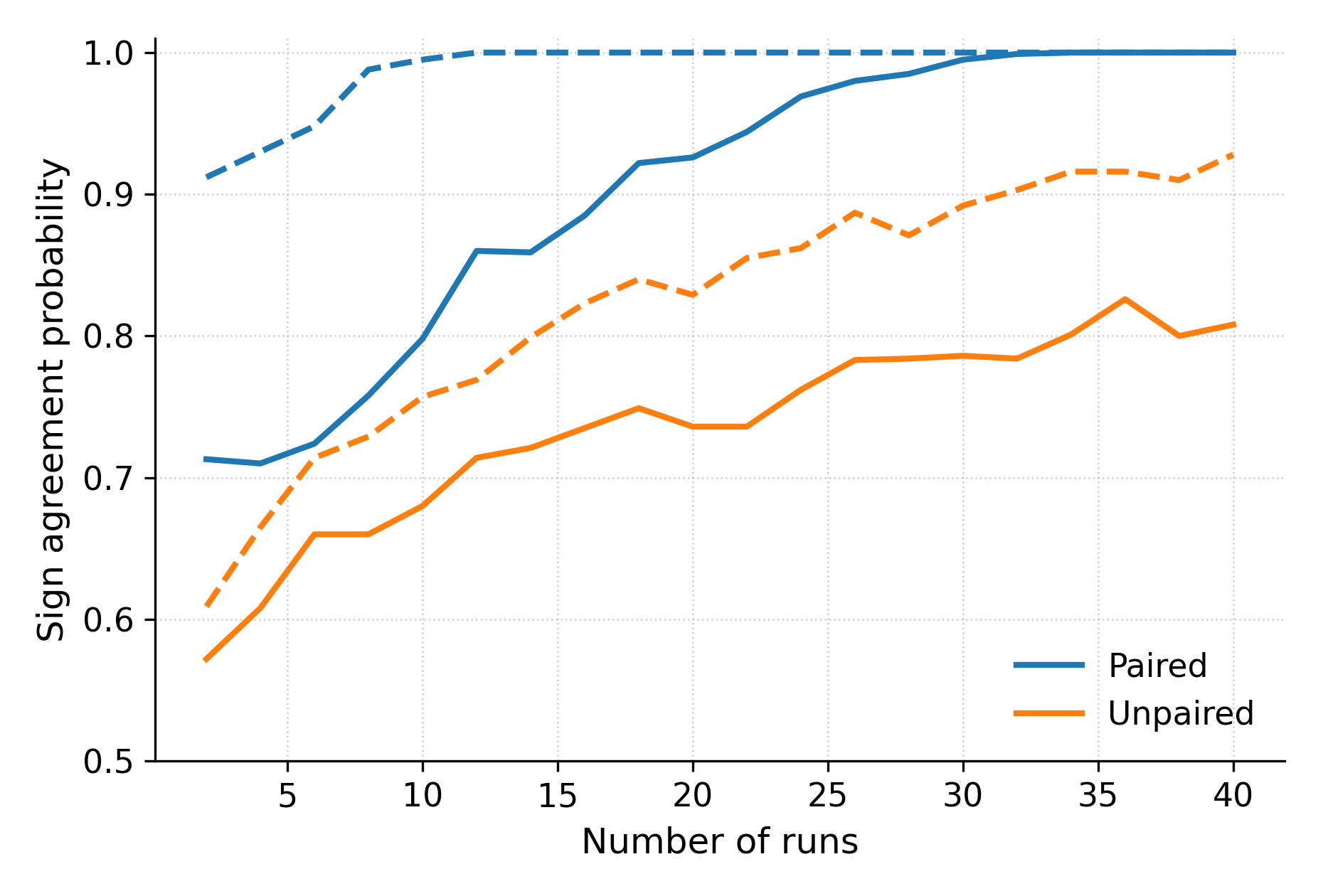}
	}
	\caption{$(a)$ 95\% confidence interval half-width. $(b)$ Statistical power for a 2\% change in wealth Gini. $(c)$ Probability of sign agreement with the estimate. Results are shown for the GDP (solid) and Gini (dashed) objectives using wealth Gini as the outcome metric.}
	\label{fig:evaluation_reliability}
\end{figure*}

\subsection{Evidence of Positive Seed-Level Correlation}
\label{subsec:positive-empirical}

We examine whether outcomes under paired evaluation exhibit positive correlation at the level of random seeds and verify this empirically using a learning-based economic simulator. The evaluation compares a baseline learned policy with an augmented policy that differs only through the inclusion of an additional policy instrument. All comparisons are conducted under controlled conditions, with differences attributable solely to the policy intervention. Full details of the simulator and experimental setup are provided in Section \ref{sec:case_study}.

Paired evaluations are conducted by evaluating both policy regimes under identical random seeds, and Pearson correlations are computed across seeds using final evaluation metrics. All objectives are evaluated using 22 shared seeds, with each seed inducing a matched stochastic realisation across policy regimes.

\begin{table}[h]
	\centering
	\caption{Seed-level correlation between paired policy outcomes}
	\label{tab:seed_correlation}
	\begin{tabular}{lcccc}
		\toprule
		Objective 
		& $\rho$ (GDP) 
		& $\rho$ (Wealth Gini)  \\
		\midrule
		GDP 
		& 0.917 
		& 0.681  \\
		
		GDP--Gini 
		& 0.993  
		& 0.963  \\
		
		Gini 
		& 0.951 
		& 0.910 \\
		\bottomrule
	\end{tabular}
	
	\vspace{0.5em}
	\footnotesize
	\textit{Notes:} All objectives use 22 seeds. Outcomes correspond to final evaluation metrics.
\end{table}

Table \ref{tab:seed_correlation} reports seed-level correlations for representative aggregate and distributional outcomes. Across all objectives and metrics, correlations are large and positive, typically exceeding 0.9, indicating strong alignment of outcomes across policy regimes at the seed level. This supports the applicability of paired seed evaluation and is consistent with the variance reductions observed relative to independent evaluation.

\subsection{Implications for Evaluation Sample Size}

This section examines how the variance reduction induced by paired seed evaluation translates into improved finite-sample inference under fixed computational budgets. We measure sample size by the number of random seeds $n$, corresponding to $r = 2n$ total simulator runs under paired evaluation.

For a fixed number of simulator runs, the standard error of the paired estimator is related to the corresponding independent standard error by

\begin{equation}
	\mathrm{SE}_{\text{pair}}
	=
	\mathrm{SE}_{\text{ind}}
	\sqrt{
		1
		-
		\frac{
			2\rho\,\sigma_1\,\sigma_0
		}{
			\sigma_1^{2}
			+
			\sigma_0^{2}
		}
	}
	\label{eq:se}
	\tag{3}
\end{equation}

which is strictly smaller whenever the seed-level correlation $\rho$ is positive. Confidence intervals constructed using paired evaluation are therefore uniformly tighter than those obtained under independent evaluation at the same computational budget. Figure \ref{fig:evaluation_reliability}$(a)$ illustrates this effect by plotting confidence interval half-widths against the number of evaluation runs under paired and independent designs. Paired evaluation consistently yields narrower intervals, with the largest gains in low-budget regimes.

By increasing the signal-to-noise ratio, paired evaluation improves statistical power and stabilises the direction of estimated effects at fixed budgets. Figures \ref{fig:evaluation_reliability}$(b)$ and \ref{fig:evaluation_reliability}$(c)$ show that paired evaluation achieves higher power and greater sign agreement with high-precision reference estimates using fewer runs. Further details are provided in Appendix \ref{app:power_disc}.

\subsection{Effective Sample Size}

The statistical advantage of paired evaluation can be summarised through an effective sample size, which translates variance reduction into an equivalent number of independently seeded simulator runs. This framing expresses the benefit of pairing directly in terms of computational cost.

Formally, we define the effective sample size $r_{\mathrm{eff}}$ as the number of independent runs required to achieve the same estimator variance as a paired evaluation based on $r$ total runs. Equating the variance of the paired and independent estimators yields
\begin{equation}
	r_{\mathrm{eff}}
	=
	r \cdot
	\frac{
		\sigma_1^{2}
		+
		\sigma_0^{2}
	}{
		\sigma_1^{2}
		+
		\sigma_0^{2}
		-
		2\rho\,\sigma_1\,\sigma_0
	}
	\label{eq:effective_sample}
	\tag{4}
\end{equation}
where $\rho$ denotes the seed-level correlation between outcomes. When $\sigma_1^{2} \approx \sigma_0^{2}$, this expression simplifies to
\begin{equation*}
	r_{\mathrm{eff}} \approx \frac{r}{1 - \rho}
\end{equation*}

Paired seed evaluation is a property of experimental design rather than of the estimator or learning algorithm. By treating the random seed as the unit of randomisation, shared stochasticity is held fixed across regimes, yielding strictly tighter uncertainty whenever seed-level correlation is non-negative. We next examine how these design considerations affect empirical inference in a learning-based economic simulator.

\section{Case Study}
\label{sec:case_study}

We illustrate the practical implications of pairing through an extension of the TaxAI simulator\citep{mi2024taxai}. We augment the government action space with an additional endogenously learned policy corresponding to a universal basic income (UBI) specified as a fraction of GDP, and compare transfer and no-transfer regimes under identical training and evaluation budgets. While this case study focuses on a single simulator and intervention type, it serves to demonstrate the methodology and motivate broader investigation of when and how paired evaluation should be applied in practice.

\subsection{Experimental Setup}

We consider two policy regimes: a baseline learned tax-and-transfer policy without universal transfers, and an augmented policy with a UBI rate, specified as a fraction of contemporaneous GDP and learned endogenously as part of the government policy. Both regimes are trained using identical learning algorithms, hyperparameters, and training horizons.

The government policy is trained under three objectives: maximizing per capita GDP, minimizing inequality via the Gini coefficient, and a hybrid GDP–Gini objective balancing efficiency and equity. Under each objective, the government jointly selects tax parameters and the UBI rate based on the chosen reward. The environment comprises 100 heterogeneous households whose interactions determine aggregate and distributional outcomes. Evaluation is performed in a stationary post-training environment, with metrics reported as time averages over a fixed horizon.

We evaluate policies under two designs: paired seed evaluation, where both regimes share identical random seeds, and unpaired evaluation, where regimes use independent seeds. Computational budgets are matched, with one paired seed corresponding to two simulator runs. Policy effects are assessed using aggregate macroeconomic outcomes, reporting mean differences and confidence intervals for each metric. Further details are provided in Appendix \ref{app:taxai_ubi}.

\begin{table*}[t]
	\centering
	\caption{Estimated policy effects under the GDP--GINI objective}
	\label{tab:gdp_gini_results}
	\begin{tabular}{lcccccc}
		\toprule
		Metric
		& $\hat{\Delta}$
		& \multicolumn{2}{c}{Paired 95\% CI}
		& \multicolumn{2}{c}{Independent 95\% CI} \\
		\cmidrule(lr){3-4} \cmidrule(lr){5-6}
		& & Lower & Upper & Lower & Upper \\
		\midrule
		Wealth Gini
		& $-0.039$
		& $-0.052$
		& $-0.027$
		& $-0.097$
		& $0.018$ \\
		
		Income Gini
		& $-0.022$
		& $-0.044$
		& $-0.000$
		& $-0.077$
		& $0.033$ \\
		
		Per-capita GDP ($\times 10^{5}$)
		& $-0.942$
		& $-3.142$
		& $1.257$
		& $-23.051$
		& $21.165$ \\
		
		Tax--GDP ratio
		& $0.045$
		& $0.008$
		& $0.081$
		& $-0.058$
		& $0.147$ \\
		\bottomrule
	\end{tabular}
	\begin{minipage}{\linewidth}
		\footnotesize
		\textit{Notes:} Fixed evaluation budget of 44 runs. Point estimates coincide across evaluation designs by construction; differences arise solely from uncertainty estimation.
	\end{minipage}
\end{table*}

\subsection{Paired and Unpaired Evaluation Outcomes}

Table~\ref{tab:gdp_gini_results} reports policy effect estimates under the GDP--Gini objective using a fixed budget of 44 simulator runs. Because both estimators compute average differences and subtraction is linear, point estimates are identical by construction; all differences arise solely from how variability is quantified.

Across outcomes, pairing yields narrower confidence intervals than independent evaluation. For inequality measures, this reduction alters inference: both wealth and income Gini coefficients exclude zero under paired design, while corresponding independent intervals include zero. Similar variance inflation appears for per-capita GDP and the tax-to-GDP ratio, where unpaired intervals are an order of magnitude wider and uninformative. These differences arise from positive seed-level correlation, which pairing exploits and independent evaluation discards. Consistent confidence interval narrowing is observed across metrics (Appendix Table \ref{tab:appdxgdp_gini_results}).

A possible explanation for the consistently high correlation between outcomes with and without UBI lies in the structure of the policy optimisation landscape. Across random initialisations, learning appears to converge to a small number of distinct fiscal regimes, corresponding to different basins of attraction, which leads to substantial variation in learned policies and outcomes across seeds. However, for a fixed seed, the set of accessible basins is likely to remain largely unchanged, so introducing UBI does not alter the regime selection mechanism but instead induces local adjustments within the same basin.

This hypothesis is supported by a policy-level distance analysis. For GDP-related outcomes, the median distance between policies learned with and without UBI under a fixed seed is substantially smaller than the variation observed across seeds. In particular, median within-seed distances in average GDP are typically below 15 percent of the corresponding cross-seed variation and fall to around 6 percent under the GDP–Gini objective. Distributional outcomes exhibit a weaker but consistent pattern, with within-seed distances in wealth Gini around 50 to 70 percent of cross-seed variation. Taken together, these results suggest that the high observed correlations primarily reflect shared convergence to the same fiscal regime at the seed level, rather than insensitivity of outcomes to the UBI intervention. Detailed statistics are reported in Appendix Table \ref{tab:appdx_policy_distance}.

\subsection{Economic Effects Revealed by Paired Seed Evaluation}

This subsection does not advance new policy claims on UBI. Its purpose is to show how reduced evaluation noise under paired seed evaluation reveals qualitative economic patterns that are obscured under unpaired evaluation.

Under pairing, coherent patterns emerge across objectives. When training targets GDP, inequality measures show no systematic change, with confidence intervals spanning zero. Under the Gini and GDP–Gini objectives, paired evaluation instead yields concentrated estimates indicating reductions in both wealth and income inequality. Unpaired evaluation fails to recover these patterns, with confidence intervals sufficiently wide that inequality effects remain indistinguishable from noise, even when point estimates align.

The same contrast appears for aggregate outcomes. Under paired design, per capita GDP effects are tightly bounded around zero across objectives, indicating limited output distortion. Unpaired evaluation admits large positive and negative effects, precluding inference. Similarly, pairing identifies modest increases in the tax-to-GDP ratio required to finance UBI, while unpaired evaluation remains uninformative. Owing to the progressive tax structure, this burden falls disproportionately on higher-income and higher-wealth households.

These results indicate that the economic regularities highlighted here are not artefacts of policy design, but of evaluation design. Paired evaluation reveals objective-aligned redistribution, fiscally feasible financing, and limited output distortion, whereas unpaired evaluation obscures these qualitative patterns under the same computational budget.

\section{Limitations}

The statistical efficiency gains from paired seed evaluation depend on the presence of non-negative seed-level correlation between outcomes under comparison. While positive correlation arises naturally in many learning-based simulators, where random seeds govern initialization, procedural generation, and large portions of the training dynamics, this condition is not guaranteed universally. When correlation is positive, pairing yields strict variance reduction, and when it is approximately zero, paired evaluation recovers standard independent evaluation. In contrast, negative seed-level correlation reverses the variance comparison, in which case paired evaluation may be less efficient than independent designs.

The magnitude of correlation may vary across evaluation metrics: structured or constrained objectives that depend on shared macro-level state variables often exhibit stronger seed-level alignment than raw rewards dominated by short-horizon or idiosyncratic noise. Importantly, correlation is only required for the specific metric being compared. While paired inference remains statistically valid regardless of correlation sign, efficiency gains from pairing arise only when seed-level correlation is non-negative.

Paired seed evaluation is most effective when random seeds control the primary sources of training variability. Environments with substantial sources of stochasticity not governed by the random seed may exhibit weaker seed-level correlation, limiting achievable variance reduction. Similarly, if compared algorithms use seeds to control fundamentally different random processes, the extent of shared structure may be limited. For pairing to deliver substantial gains, it is therefore important that experimental infrastructure ensures reproducibility and that seeds govern overlapping aspects of the learning process across all conditions under comparison.

\section{Conclusion}

This paper shows that comparative evaluation in learning-based simulators can be improved by exploiting shared sources of randomness across alternative interventions. When outcomes under competing treatments are generated using identical random seeds, they often exhibit positive correlation arising from common initial conditions, stochastic shocks, and overlapping learning trajectories. Treating the random seed as the unit of randomisation and reusing it across alternatives leverages this structure to reduce estimator variance when seed-level outcomes are positively correlated, without increasing the number of simulator runs.

We showcase how paired evaluation designs can reduce variance via seed reuse, improving confidence intervals, statistical power, and effective sample size under fixed computational budgets. When seed-level correlation is weak or absent, paired evaluation recovers standard independent evaluation, while under non-negative correlation it yields lower estimator variance. These results highlight evaluation design as an important determinant of reliability in learning-based simulation studies.

While the empirical analysis is based on a single simulator and intervention type, it serves primarily to demonstrate the methodology and clarify the conditions under which paired evaluation is effective. The broader applicability of paired seed evaluation across different simulators, training pipelines, and metrics remains an open question. In particular, understanding how modern sources of stochasticity interact with seed control, and whether paired designs can be incorporated into standard benchmarking protocols, are natural directions for further investigation.

\clearpage

\normalsize
\bibliography{references}

\clearpage
\appendix
\section{Appendix}
\label{app:proofs_verifications}

\subsection{Proof of Theorem 1}
\label{app:proof_t1}

\begin{proof}
	Under independent evaluation, the two sample means are computed using independent random seeds. Therefore,
	\begin{equation*}
		\operatorname{Var}\!\left( \hat{\Delta}_{\text{ind}} \right)
		=
		\frac{1}{n}\operatorname{Var}\!\left( Y(1,s) \right)
		+
		\frac{1}{n}\operatorname{Var}\!\left( Y(0,s) \right)
	\end{equation*}
	
	Under paired evaluation,
	\begin{equation*}
		\operatorname{Var}\!\left( \hat{\Delta}_{\text{pair}} \right)
		=
		\frac{1}{n}\operatorname{Var}\!\left( Y(1,s) - Y(0,s) \right)
	\end{equation*}
	
	Expanding the variance,
	\begin{equation*}
		\begin{aligned}
			\operatorname{Var}\!\left( Y(1,s) - Y(0,s) \right)
			&=
			\operatorname{Var}\!\left( Y(1,s) \right)
			+
			\operatorname{Var}\!\left( Y(0,s) \right) \\
			&\quad
			-
			2\,\operatorname{Cov}\!\left( Y(1,s), Y(0,s) \right)
		\end{aligned}
	\end{equation*}
	
	Dividing by $n$ yields the stated result.
\end{proof}

\subsection{Standard Error and Estimation Details}
\label{app:standard_error}

To estimate finite-sample standard errors empirically, we employ a Monte Carlo subsampling procedure. For each subsample size $n$, subsets of $n$ seeds are repeatedly drawn from the available evaluation runs, and the policy effect is re-estimated under both paired and independent designs. The empirical standard deviation of these estimates provides an estimate of the sampling standard error at that sample size. This procedure underlies the confidence interval estimates reported in the main text.

A paired evaluation based on $r$ total simulator runs corresponds to $r/2$ paired seeds. The variance of the paired sample mean estimator is therefore
\[
\operatorname{Var}_{\text{pair}}
=
\frac{
	\sigma_1^2 + \sigma_0^2 - 2\rho\,\sigma_1\sigma_0
}{
	r/2
}
\]

Under independent evaluation, the two outcomes are generated using independent seeds, implying
\[
\operatorname{Var}_{\text{ind}}
=
\frac{
	\sigma_1^2 + \sigma_0^2
}{
	r/2
}
\]

Taking square roots yields the corresponding standard errors. Expressing the paired standard error relative to the independent one gives
\[
\mathrm{SE}_{\text{pair}}
=
\mathrm{SE}_{\text{ind}}
\sqrt{
	1
	-
	\frac{
		2\rho\,\sigma_1\,\sigma_0
	}{
		\sigma_1^2 + \sigma_0^2
	}
}
\]
which establishes Equation~\ref{eq:se}. Positive seed-level correlation therefore directly reduces estimator uncertainty under paired evaluation. All expectations and variances are taken with respect to the distribution over random seeds, which are assumed independent across pairs.

\subsection{Detection and Directional Reliability}
\label{app:power_disc}

Directional stability is defined as the probability of correctly identifying the sign of the treatment effect. Let $\hat{\Delta}_{n}$ denote the estimated effect computed using a subsample of $2n$ runs, and let $\hat{\Delta}_{\mathrm{full}}$ denote the estimate obtained using all available seeds. We measure directional stability as
\begin{equation*}
	\Pr\!\left(
	\operatorname{sign}\!\left(\hat{\Delta}_{n}\right)
	=
	\operatorname{sign}\!\left(\hat{\Delta}_{\mathrm{full}}\right)
	\right)
\end{equation*}
where the full-sample estimate serves as a high-precision reference for the underlying effect direction.

Directional stability and statistical power are governed by the same underlying quantity, namely the signal to noise ratio $|\Delta|/\mathrm{SE}$. For a two sided hypothesis test at significance level $\alpha$, statistical power can be approximated as

\begin{equation*}
	\text{Power}
	\approx
	1 - \Phi\!\left(
	z_{1-\alpha/2} - \frac{|\Delta|}{\mathrm{SE}}
	\right)
	\label{eq:power}
\end{equation*}

where $\Phi(\cdot)$ denotes the standard normal cumulative distribution function.

\subsection{Effective Sample Size}

This variance reduction admits an equivalent interpretation in terms of effective sample size. Let $r_{\mathrm{eff}}$ denote the number of independently seeded runs required to achieve the same variance as a paired evaluation based on $r$ total runs. An independent evaluation using $r_{\mathrm{eff}}$ runs has variance
\[
\operatorname{Var}_{\text{ind}}
=
\frac{
	\sigma_1^2 + \sigma_0^2
}{
	r_{\mathrm{eff}}/2
}
\]

Equating this expression with $\operatorname{Var}_{\text{pair}}$ and solving for $r_{\mathrm{eff}}$ yields
\[
r_{\mathrm{eff}}
=
r \cdot
\frac{
	\sigma_1^2 + \sigma_0^2
}{
	\sigma_1^2 + \sigma_0^2 - 2\rho\,\sigma_1\,\sigma_0
}
\]
which establishes Equation~\ref{eq:effective_sample}. This expression quantifies the computational gains from pairing in terms of equivalent independent simulator runs.

\section{Appendix - Case Study Details}

\renewcommand{\thetable}{B.\arabic{table}}
\setcounter{table}{0}

\subsection{The TaxAI Framework \citep{mi2024taxai}} 
\label{app:taxai_ubi}

TaxAI is a dynamic Bewley–Aiyagari economy implemented as a partially observable multi-agent reinforcement-learning environment. Households and the government learn behavioural policies over time, rather than solving analytical optimization problems. This makes TaxAI suitable for studying policy interventions whose effects depend on bounded rationality, information frictions, and adaptive decision-making, which directly motivates our UBI extension. Code can be found at \url{https://github.com/DemandredEng/TaxAI}.

Each household chooses consumption, savings, and labour supply while facing idiosyncratic productivity shocks and non-linear taxes. Household income at time $t$ is given by:
\begin{equation*}
	i_t = W_t\, e_t\, h_t + r_{t-1}\, a_t
\end{equation*}
where $W_t$ is the wage rate, $e_t$ labour productivity, $h_t$ working hours, $a_t$ asset and return to savings $r_{t-1}$. Households cannot borrow $(a_t >= 0)$. To reflect realistic information frictions, TaxAI compresses household-level information into two coarse bins; the top 10\% richest and bottom 50\% poorest. Agents observe their own productivity and assets, along with grouped averages of income, wealth, and productivity, while the government observes economy-wide aggregates. This partial-observability structure shapes how learned policies respond to macroeconomic conditions.

Households face income taxes $T(i_t)$, wealth taxes $T^a(a_t)$, and a proportional consumption tax rate $\tau_s$. Their period budget constraint is:
\begin{equation*}
	(1+\tau_s)c_t + a_{t+1}
	=
	i_t - T(i_t) + a_t - T^a(a_t)
\end{equation*}
This constraint defines the mapping between government policy and household behaviour. The specific tax functions follow the non-linear HSV formulation from the original TaxAI model, and we refer the reader to the original work for full details.

The government learns a policy over non-linear income-tax parameters, wealth-tax parameters, consumption-tax rates, and spending. Prices are determined each period by a representative firm with Cobb–Douglas production and a financial intermediary that converts savings into capital. These components are unchanged and provide the macroeconomic backbone for our extension.

This baseline defines behavioural frictions and macro feedbacks that our UBI mechanism must operate within. In the next section, we introduce an endogenous UBI policy that is jointly learned with the tax parameters, enabling a controlled comparison between economies with and without transfers.

\subsection{UBI Mechanism}

This section introduces a universal basic income (UBI) mechanism into the TaxAI environment described previously. The extension is deliberately minimal; it preserves the original market structure, household decision problem, and learning setup, while adding a single redistributive instrument that operates through lump-sum transfers. This design enables clean comparisons between economies with and without transfers, holding all other institutional features fixed.

We augment the government’s action space with a scalar policy parameter $\phi_t \in [0, \bar{\phi}]$, representing the fraction of contemporaneous gross domestic product allocated to UBI. The UBI budget at time $t$ is given by
\begin{equation*}
	\mathcal{U}_t = \phi_t Y_t
\end{equation*}
where $Y_t$ denotes aggregate output produced by the representative firm. The government selects $\phi_t$ jointly with the existing tax and spending parameters as part of its reinforcement-learning policy. The upper bound $\bar{\phi}$ is fixed exogenously to ensure fiscal feasibility and stabilize learning.

UBI is distributed uniformly across households. Let $N$ denote the number of households. The per-household transfer implied by policy $\phi_t$ is
\begin{equation*}
	u_t = \frac{\mathcal{U}_t}{N}
\end{equation*}
To avoid contemporaneous feedback loops and to reflect realistic administrative delays, transfers are paid with a one-period lag. A UBI decision made at time $t$ is disbursed at time $t+1$.

Upon receipt at time $t+1$, the UBI transfer is added to household assets after income and asset taxes have been applied, but before households choose consumption and savings. The household budget constraint therefore becomes
\begin{equation*}
	(1+\tau_s)c_t + a_{t+1}
	=
	i_t - T(i_t) + a_t - T^a(a_t) + u_{t-1}
\end{equation*}
where $u_{t-1}$ denotes the per-household UBI payment determined in the previous period.

By construction, the transfer is lump-sum and does not depend on individual income, assets, or labour supply. It therefore does not alter within-period marginal incentives, though it affects future behaviour through wealth accumulation. The transfer augments household wealth directly and is subject to standard taxation only from period $t+2$ onward, once it becomes part of the endogenous asset state.

The government’s inter-temporal budget constraint is modified to account for UBI disbursements. Let $B_t$ denote government debt. Total UBI paid at time $t$ is equal to $\mathcal{U}_{t-1}$, and enters the budget constraint symmetrically with other expenditures:
\begin{equation*}
	(1+r_{t-1})B_t + G_t + \mathcal{U}_{t-1}
	=
	B_{t+1} + T_t
\end{equation*}
where $T_t$ denotes total tax revenue from income, wealth, and consumption taxes as defined above. No additional financing instrument is introduced; UBI is funded implicitly through taxation and government debt. No additional financing instrument is introduced; UBI is financed through the government’s existing tax and debt instruments.

The UBI parameter $\phi_t$ is selected under the same partial-observability constraints as other fiscal instruments. The government observes grouped macroeconomic aggregates but does not condition transfers on individual-level states. Households do not observe $\phi_t$ directly; instead, they infer policy effects through realized transfers and aggregate outcomes. This preserves the informational structure of TaxAI.

In the baseline TaxAI environment, episodes terminate whenever aggregate household consumption and government expenditure exceed contemporaneous output, enforcing a hard feasibility constraint. Under universal transfers, this rule induces frequent premature termination and prevents meaningful policy evaluation, as exploratory policies involving UBI are truncated before their longer-run effects can be observed.

To maintain incentives towards feasibility while allowing learning to proceed, we replace hard termination with a strongly binding penalty applied to the government’s reward whenever aggregate demand exceeds output. The penalty is proportional to normalized excess demand and is smoothly bounded to ensure numerical stability. Formally, define the excess-demand penalty $\kappa_t$ and the corresponding government reward $r^{\text{gov}}_t$ adjustment as:

\begin{equation*}
	\kappa_t
	= \max\!\left(0,\; \frac{C_t + G_t - Y_t}{Y_t}\right)
\end{equation*}

\begin{equation*}
	r^{\text{gov}}_t
	\leftarrow r^{\text{gov}}_t - \lambda_\kappa \tanh(\kappa_t)
\end{equation*}

where $\lambda_\kappa > 0$ is a fixed penalty weight that governs the severity of the feasibility violation in the government’s reward. Empirically, the excess-demand penalty becomes inactive after early training for most runs, indicating that the learned policy internalizes the resource constraint rather than relying on infeasible states.

This extension treats UBI as a universal, unconditional transfer that redistributes wealth and smooths consumption without directly distorting labour or savings decisions. The government is allowed to use UBI as an additional policy lever alongside existing fiscal instruments, but the analysis focuses on how the availability of this instrument alters behaviour and aggregate outcomes, rather than on the composition of the learned policy rules themselves. Because the only difference between experimental economies is the presence of the UBI instrument, and because identical random seeds are used across runs, the framework permits paired seed evaluation of distributional, behavioural, and macroeconomic outcomes.

\subsection{Agents and Learning Setup}

The learning framework closely follows the original TaxAI design; the only modification is the inclusion of the UBI policy parameter in the government’s action space. All agents are trained using the Bi-level Mean-Field Actor–Critic (BMFAC) algorithm, which we employ consistently across all experiments to ensure comparability.

The economy is modelled as a partially observable multi-agent Markov game with two types of strategic agents: a single government agent and a large population of heterogeneous household agents. The government learns fiscal policy parameters, while households learn individual consumption, savings, and labour-supply strategies. Due to the large number of households, direct centralized training over the full joint state space is infeasible. Following TaxAI, we therefore adopt a mean-field approximation in which households interact with aggregate group-level statistics rather than the full population state\citep{yang2018meanfield}.

\subsection{Objectives and Reward Functions}

The reward structure is identical to that of TaxAI. Household agents maximize discounted lifetime utility, given by a CRRA utility function over consumption minus disutility from labour supply. No additional reward shaping or penalties are introduced, and the UBI transfer affects household behaviour only through the budget constraint.

The government’s objective is task-dependent. We consider three of the four reward specifications defined in the TaxAI framework: (i) maximizing GDP growth, (ii) minimizing inequality, and (iii) a multi-objective combination of maximising GDP and minimising inequality. For each reward specification, we conduct paired experiments with and without UBI, holding the learning algorithm and training protocol fixed. This design allows us to assess how the presence of UBI alters outcomes under different policy objectives, rather than conditioning results on a single normative criterion.

\subsection{Action Spaces and Constraints}

The household action space is unchanged from TaxAI. Each household selects a savings ratio and labour supply level, which jointly determine consumption and next-period assets subject to the household budget constraint.

The government’s action space includes all fiscal instruments defined in TaxAI; non-linear income-tax parameters, non-linear wealth-tax parameters, and the government spending-to-GDP ratio, augmented by a single additional scalar action corresponding to the UBI share of GDP. Thus, the dimensionality of the government action space increases by one relative to the baseline environment. All feasibility constraints, action bounds, and proportional-action formulations remain unchanged, ensuring that fiscal policies respect the same budgetary structure as in the original model.

\subsection{Training Procedure and Stability}

Training follows the same protocol as TaxAI. Government and household agents learn simultaneously using BMFAC, with decentralized execution and critics conditioned on mean-field aggregate information. Exploration, action clipping, reward normalization, and other stabilization mechanisms are inherited directly from the baseline framework. No new regularization terms or constraint-handling mechanisms are introduced specifically for UBI. 

\begin{table*}[h]
	\centering
	\caption{Estimated policy effects at a fixed evaluation budget of 44 runs.}
	\label{tab:appdxgdp_gini_results}
	\begin{tabular}{lcccccc}
		\toprule
		Metric
		& $\hat{\Delta}$
		& \multicolumn{2}{c}{Paired 95\% CI}
		& \multicolumn{2}{c}{Independent 95\% CI} \\
		\cmidrule(lr){3-4} \cmidrule(lr){5-6}
		& & Lower & Upper & Lower & Upper \\
		\midrule
		
		\multicolumn{6}{l}{\textbf{GDP Objective}} \\
		\addlinespace[0.3em]
		
		Wealth Gini
		& $-0.0238$
		& $-0.0536$
		& $0.0060$
		& $-0.0749$
		& $0.0273$ \\
		
		Income Gini
		& $0.0006$
		& $-0.0307$
		& $0.0319$
		& $-0.0577$
		& $0.0589$ \\
		
		Per-capita GDP ($\times 10^{5}$)
		& $1.4761$
		& $-1.7099$
		& $4.6620$
		& $-9.1778$
		& $12.1299$ \\
		
		Tax--GDP ratio
		& $0.0426$
		& $-0.0302$
		& $0.1155$
		& $-0.0902$
		& $0.1755$ \\
		
		\addlinespace[0.6em]
		\multicolumn{6}{l}{\textbf{GINI Objective}} \\
		\addlinespace[0.3em]
		
		Wealth Gini
		& $-0.0339$
		& $-0.0486$
		& $-0.0192$
		& $-0.0805$
		& $0.0127$ \\
		
		Income Gini
		& $0.0024$
		& $-0.0210$
		& $0.0259$
		& $-0.0516$
		& $0.0565$ \\
		
		Per-capita GDP ($\times 10^{5}$)
		& $2.1004$
		& $-1.1896$
		& $5.3904$
		& $-12.0359$
		& $16.2367$ \\
		
		Tax--GDP ratio
		& $-0.0069$
		& $-0.0756$
		& $0.0617$
		& $-0.1206$
		& $0.1067$ \\
		
		\bottomrule
	\end{tabular}
\end{table*}

\begin{table*}[h]
	\centering
	\caption{Within-seed and cross-seed policy distances across tasks and metrics}
	\label{tab:appdx_policy_distance}
	\begin{tabular}{lccc}
		\toprule
		Metric
		& Within-seed median
		& Cross-seed distance
		& Median ratio \\
		\midrule
		
		\multicolumn{4}{l}{\textbf{GDP Objective}} \\
		\addlinespace[0.3em]
		
		Per-capita GDP ($\times 10^{5}$)
		& $2.17$
		& $14.26$
		& $0.152$ \\
		
		Wealth Gini
		& $0.0388$
		& $0.0722$
		& $0.537$ \\
		
		\addlinespace[0.6em]
		\multicolumn{4}{l}{\textbf{GDP--Gini Objective}} \\
		\addlinespace[0.3em]
		
		Per-capita GDP ($\times 10^{5}$)
		& $1.34$
		& $21.95$
		& $0.061$ \\
		
		Wealth Gini
		& $0.0405$
		& $0.0771$
		& $0.525$ \\
		
		\addlinespace[0.6em]
		\multicolumn{4}{l}{\textbf{Gini Objective}} \\
		\addlinespace[0.3em]
		
		Per-capita GDP ($\times 10^{5}$)
		& $1.64$
		& $18.85$
		& $0.087$ \\
		
		Wealth Gini
		& $0.0412$
		& $0.0611$
		& $0.674$ \\
		
		\bottomrule
	\end{tabular}
	\begin{minipage}{\linewidth}
		\footnotesize
		\textit{Notes:} Within-seed medians report the distance between policies learned with and without UBI under identical random seeds. 
		Cross-seed distances measure variation across different random initialisations. 
		The median ratio is defined as the ratio of within-seed median distance to cross-seed distance. 
		Ratios substantially below one indicate that UBI induces local policy adjustments within a shared fiscal regime rather than shifts across regimes. 
		Cross-seed distance is measured as the mean absolute deviation of outcomes from their task-level mean, pooling both UBI and no-UBI runs.
	\end{minipage}
\end{table*}

\end{document}